\documentclass[12pt]{elsarticle}

\usepackage{amsmath,mathtools,amssymb}
\usepackage{subcaption}
\usepackage{tikz}
\usepackage{soul}
\usepackage[framemethod=tikz]{mdframed}
\usepackage{pgfplots}
\usetikzlibrary{shapes,arrows,shapes.geometric,fit,matrix,backgrounds,plotmarks}
\PassOptionsToPackage{table}{xcolor}
\usepackage{graphicx}
\usepackage{wasysym}
\usepackage{algorithm}
\usepackage{algorithmicx}
\usepackage[noend]{algpseudocode}
\newtheorem{myex}{Example}

\newtheorem{mydef}{Def.}
\newtheorem{myprob}{Problem}

\newtheorem{mythm}{Theorem}
\newtheorem{proof}{Proof}
\newtheorem{myprop}{Proposition}
\newtheorem{mycor}{Corollary}
\newtheorem{mylemma}{Lemma}
\newtheorem{myremark}{Remark}

\usepackage{setspace}
\usepackage{framed}
\usepackage{fancybox}
\usepackage{comment}
\usepackage[linktocpage=true,breaklinks]{hyperref}

\definecolor{cornellred}{rgb}{0.7, 0.11, 0.11}
\definecolor{darkblue}{rgb}{0.0, 0.0, 0.55}

\usepackage{lineno}

\journal{Artificial Intelligence}

\begin{document}

\begin{frontmatter}



\title{A note on the complexity of the causal ordering problem}


\author[ibm]{Bernardo Gon\c{c}alves}
\ead{begoncalves@acm.org}
\author[lncc]{Fabio Porto}
\ead{fporto@lncc.br}
\address[ibm]{IBM Research, S\~ao Paulo, Brazil}
\address[lncc]{National Laboratory for Scientific Computing (LNCC), Petr\'opolis, Brazil}

\begin{abstract}
In this note we provide a concise report on the complexity of the causal ordering problem, originally introduced by Simon to reason about causal dependencies implicit in systems of mathematical equations. We show that Simon's classical algorithm to infer causal ordering is NP-Hard---an intractability previously guessed but never proven. We present then a detailed account based on Nayak's suggested algorithmic solution (the best available), which is dominated by computing transitive closure---bounded in time by $O(|\mathcal V|\cdot |\mathcal S|)$, where $\mathcal S(\mathcal E, \mathcal V)$ is the input system structure composed of a set $\mathcal E$ of equations over a set $\mathcal V$ of variables with number of variable appearances (density) $|\mathcal S|$. We also comment on the potential of causal ordering for emerging applications in large-scale hypothesis management and analytics. 

\end{abstract}


\end{frontmatter}



\section{Introduction}\label{sec:intro}
\noindent
The causal ordering problem has long been introduced by Simon as a technique to infer the causal dependencies implicit in a deterministic mathematical model \cite{simon1953}. For instance, let $f_1(x_1)$ and $f_2(x_1, x_2)$ be two equations defined over variables $x_1, x_2$. Then the causal ordering problem is to infer all existing causal dependencies, in this case the only one is $(x_1, x_2)$, read `$x_2$ causally depends on $x_1$.' 
It is obtained by first matching each equation to a variable that appears in it, e.g., $f_2 \mapsto x_2$. Intuitively, this means that $f_2$ is to be assigned to compute the value of $x_2$---using the value of $x_1$, which establishes a direct causal dependency between these two variables. Indirect dependencies may then arise and can be computed, which is specially useful when the system of equations is very large. 

Causal ordering inference can then support users with uncertainty management, say, towards the discovery of what is wrong with a model for enabling efficient and effective modeling intervention. If multiple values of $x_1$ are admissible for a modeler, then as a user of the causal ordering machinery she has support to track their influence on the values of $x_2$. One major application for that is probabilistic database design \cite{goncalves2014}.

Back in the 50th's, Simon was motivated by studies in econometrics and not very concerned with the algorithmic aspects of the Causal Ordering Problem (COP). Yet his vision on COP and its relevance turned out to be influential in the artificial intelligence literature. 
In a more recent study, Dash and Druzdzel revisit and motivate it in light of modern applications \cite{druzdzel2008}. They show that Simon's original algorithm, henceforth the Causal Ordering Algorithm (COA), is correct in the sense that any valid causal ordering that can be extracted from a self-contained (well-posed) system of equations must be compatible with the one that is output by COA \cite{druzdzel2008}. Their aim has also been (sic.) to validate decades of research that has shown the causal ordering to provide a powerful tool for operating on models. In addition to the result on the correctness of COA, their note also provides a convenient survey of related work that connects to Simon's early vision on causal reasoning. 

However, Simon's formulation of COP into COA---originally in \cite{simon1953}, and reproduced in \cite{druzdzel2008}, turns out to be intractable. 
There is a need to distinguish Simon's COA from COP itself. 
The former still seems to be the main entry point to the latter in the specialized literature. 
In fact, there is a lack of a review on the computational properties of COA---and as we show in this note, it tries to address an NP-Hard problem as one of its steps. 
The interested reader who needs an efficient algorithmic approach to address COP in a real, large-scale application can only scarcely find some comments spread through Nayak \cite[p. 287-91]{nayak1994}, and then Iwasaki and Simon \cite[p. 149]{simon1994} and Pearl \cite[p. 226]{pearl2000} both pointing to the former. Regarding Simon's COA itself, the classical approach to COP, it is only Nayak who suggests in words that (sic.) `[it] is a worst-case exponential time algorithm' \cite[p. 37]{nayak1996}. We discuss this ambiguity that exists in the most up-to-date literature shortly in \S\ref{subsec:related-work}. 

COP is significant also in view of emerging applications in large-scale hypothesis management and analytics \cite{goncalves2014}. The modeling of physical and socio-economical systems as a set of mathematical equations is a traditional approach in science and engineering and a very large bulk of models exist which are ever more available in machine-readable format. Simon's early vision on the automatic extraction of the ``causal mechanisms'' implicit in (large-scale) models for the sake of informed intervention finds nowadays new applications in the context of open simulation laboratories \cite{goncalves2015cise}, large-scale model management \cite{haas2011} and online, shared model repositories \cite{hunter2003,hines2004,chelliah2013}.

In this paper we review the causal ordering problem (\S\ref{sec:problem}). Our core contributions are (\S\ref{sec:coa}) to originally show that COA aims at addressing an NP-Hard problem, confirming Nayak's earlier intuition; and then (\S\ref{sec:nayak}) to organize into a concise yet complete note his hints to solve COP in polynomial time.

\subsection{Informal Preliminaries}\label{subsec:preliminaries}
\noindent
Given a system of mathematical equations involving a set of variables, the \emph{causal ordering problem} consists in detecting the hidden asymmetry between variables. As an intermediate step towards it, one needs to establish a one-to-one mapping between equations and variables \cite{simon1953}. 

For instance, Einstein's famous equation $E=m\,c^2$ states the equivalence of mass and energy, summarizing (in its scalar version) a theory that can be imposed two different asymmetries for different applications. Say, given a fixed amount of mass $m=m_0$ (and recalling that $c$ is a constant), find the particle's relativistic rest energy $E$; or rather, given the particle's rest energy, find its mass or potential for nuclear fission. That is, the causal ordering depends on what variables are set as input and which ones are ``influenced'' by them. Suppose there is uncertainty, say, a user considers two values to set the mass, $\!m=m_0$ or $m= m_0^\prime$. $\!$Then this uncertainty will flow through the causal ordering and affect all variables that are dependent on it (energy $E$). 

For really large systems, having structures say in the order of one million equations \cite{goncalves2015c}, the causal ordering problem is critical to provide more specific accountability on the accuracy of the system---viz., what specific variables and subsystems account for possibly inaccurate outcomes. This is key for managing and tracking the uncertainty of alternative modeling variations systematically \cite{goncalves2015cise,goncalves2015c}.

\subsection{Related Work}\label{subsec:related-work}
\noindent

\textbf{COA}. Dash and Druzdzel \cite{druzdzel2008} provide a high-level description of how Simon's COA \cite{simon1953} proceeds to discover the causal dependencies implicit in a structure. It returns a `partial' causal mapping: from partitions on the set of equations to same-cardinality partitions on the set of variables---a `total' causal mapping would instead map every equation to exactly one variable.  

They show then that any valid total causal mapping produced over a structure must be consistent with COA's partial causal mapping. 
Nonetheless, no observation at all is made regarding COA's computational properties in \cite{druzdzel2008}, leaving in the most up-to-date literature an impression that Simon's COA is the way to go for COP. In this note we show that Simon's COA tries to address an NP-Hard problem in one of its steps, and then clearly recommend Nayak's efficient approach to COP as a fix to COA.

\textbf{COP}. Inspired by Serrano and Gossard's work on constraint modeling and reasoning \cite{serrano1987}, Nayak describes an approach that is provably efficient to process the causal ordering: extract any valid total causal mapping and then compute the transitive closure of the direct causal dependencies, viz, the causal ordering. Nayak's is a provably correct approach to COP, as all valid `total' causal mappings must have the same causal ordering. 

In this note we arrange those insights into a concise yet detailed recipe that can be followed straightforwardly to solve COP efficiently.

\section{The Causal Ordering Problem}\label{sec:problem}
\noindent
We start with some preliminaries on notation and basic concepts to eventually state COP formally.

For an equation $f(x_1, x_2, ..., x_\ell)=0$, we will write $Vars(f) \triangleq \{ x_1,\, x_2,\, ...,\, x_\ell \}$ to denote the set of variables that appear in it.

\begin{mydef}\label{def:structure}
A \textbf{structure} is a pair $\mathcal S(\mathcal E, \mathcal V)$, where $\mathcal E$ is a set of equations over a set $\mathcal V\!$ of variables, $\mathcal V \triangleq \bigcup_{f \in\, \mathcal E} Vars(f)$, such that:
\begin{itemize}
\item[(a)] In any subset $\mathcal E^\prime \subseteq \mathcal E$ of $k>0$ equations of the structure, at least $k$ different variables appear, i.e., $|\mathcal E^\prime| \leq |\mathcal V^\prime|$;

\item[(b)] In any subset of $k>0$ equations in which $r$ variables appear, $k \leq r$, if the values of any $(r - k)$ variables are chosen arbitrarily, then the values of the remaining $k$ variables can be determined uniquely---finding these unique values is a matter of solving the equations.
\end{itemize}
\end{mydef}

Note in Def.~\ref{def:structure} that structures are composed of equations, and variables are only part of them indirectly as part of equations. Accordingly, set inclusion and all set operations such as union, intersection and difference are computed w.r.t. the equations. That is, if $\mathcal S(\mathcal E, \mathcal V)$ and $\mathcal S^\prime(\mathcal E^\prime, \mathcal V^\prime)$ are structures, then we write $\mathcal S^\prime \!\subset \mathcal S$ when $\mathcal E^\prime \!\subset \mathcal E$. 
An additional operation for `variables elimination' shall be used. We write $\mathcal T := \mathcal S \div \mathcal S^\prime$, to denote a structure $\mathcal T$ resulting from both (i) removing equations $\mathcal E^\prime$ from $\mathcal E$, and (ii) enforcing elimination of variables $\mathcal V^\prime = \bigcup_{f \in \mathcal E^\prime} Vars(f)$ from $\mathcal E \setminus \mathcal E^\prime$.

\begin{mydef}\label{def:complete}
Let $\mathcal S(\mathcal E, \mathcal V)$ be a structure. We say that $\mathcal S$ is self-contained or \textbf{complete} if $|\mathcal E|=|\mathcal V|$.
\end{mydef}
\noindent
In short, COP will be concerned with systems of equations that are `structural' (Def. \ref{def:structure}) and `complete' (Def. \ref{def:complete}), viz., that have as many equations as variables and no subset of equations has fewer variables than equations.\footnote{Also, for inferring causal ordering the systems of equations given as input is expected to be `independent,' i.e., can only have non-redundant equations.}

Now Def.~\ref{def:tcm} introduces a data structure that is an intermediate product towards addressing COP. We shall state COP formally in the sequel.

\begin{mydef}\label{def:tcm}
$\!$Let $\mathcal S(\mathcal E, \mathcal V)\!$ be a complete structure. $\!\!$Then a \textbf{total causal mapping} over $\mathcal S$ is a bijection $\varphi\!: \mathcal E \to \mathcal V$ such that, for all $f \!\in \mathcal E$, if $\varphi(f)=x$ then $x \!\in Vars(f)$.
\end{mydef}

Note that such total causal mapping $\varphi$ induces a set $C_{\varphi}$ of \emph{direct causal dependencies} (see Eq. \ref{eq:direct-causal-dependencies}), which shall give us the \emph{causal dependencies} (Def.~\ref{def:causal-dependency}). 
\begin{eqnarray}
\!\!\!C_{\varphi} \!= \{\, (x_a, x_b) \,| \;\text{there exists}\; f \!\in \mathcal E \;\text{such that}\; \varphi(f) = x_b \;\text{and}\; x_a \in Vars(f)  \,\}
\label{eq:direct-causal-dependencies}
\end{eqnarray}
\vspace{-15pt}
\begin{mydef}\label{def:causal-dependency}
Let $\mathcal S(\mathcal E, \mathcal V)$ be a structure with variables $x_a, x_b \in \mathcal V$, and $\varphi$ a total causal mapping over $\mathcal S$ inducing set of direct causal dependencies $C_{\varphi}$ and indirectly a transitive closure $C^+_{\varphi}$. We say that $(x_a, x_b)$ is a \textbf{direct causal dependency} in $\mathcal S$ if $(x_a, x_b) \in C_{\varphi}$, and that $(x_a, x_b)$ is a \textbf{causal dependency} in $\mathcal S$ if $(x_a, x_b) \in C^+_{\varphi}$.
\end{mydef}

In other words, $(x_a, x_b)$ is in $C_{\varphi}$ iff $x_b$ direct and causally depends on $x_a$, given the causal asymmetries induced by $\varphi$. Then by transitive reasoning on $C_{\varphi}$ we shall be able to infer the transitive closure $C_{\varphi}^+$, which is the \emph{causal ordering}. Now we can state COP more formally as Problem~\ref{prob:cop}.

\begin{myprob}\label{prob:cop}
\emph{(COP)}. Given a complete structure $\mathcal S(\mathcal E, \mathcal V)$, find a total causal mapping $\varphi$ over $\mathcal S$ and derive a set $C_\varphi^+$ of causal dependencies induced by it.
\end{myprob}

In the sequel we shall see two different algorithmic approaches to COP (Problem~\ref{prob:cop}). First, the classical approach informally described by Simon in the 50th's \cite{simon1953}, and reproduced recently in \cite{druzdzel2008}; and then Nayak's one proposed in the 90th's \cite{nayak1994}. 
We shall present the algorithms and analyze their corresponding complexities.


\section{Simon's Causal Ordering Algorithm and its Complexity}\label{sec:coa}
\noindent
We introduce now additional concepts that are specific to Simon's COA.

\begin{mydef}
\label{def:minimal}
Let $\mathcal S$ be a structure. We say that $\mathcal S$ is \textbf{minimal} if it is complete and there is no complete substructure $\mathcal S^\prime \!\subset \mathcal S$. 
\end{mydef}

\begin{myex}
Consider structure $\mathcal S(\mathcal E, \mathcal V)$, where $\mathcal E \!=\! \{\, f_1(x_1),\; f_2(x_2),\; f_3(x_3),$ $f_4(x_1, x_2, x_3, x_4, x_5),\; f_5(x_1, x_3, x_4, x_5),\; f_6(x_4, x_6),\; f_7(x_5, x_7) \,\}$. Note that $\mathcal S$ is complete, as $|\mathcal E|\!=\!|\mathcal V|\!=\!7$, but not minimal. There are exactly three minimal substructures $\mathcal{S}_1, \mathcal{S}_2, \mathcal{S}_3 \subset \mathcal S$, whose sets of equations are $\mathcal{E}_1 \!=\! \{f_1(x_1)\},\, \mathcal{E}_2 \!=\! \{f_2(x_2)\},\,\mathcal{E}_3 \!=\! \{f_3(x_3)\}$. $\Box$
\label{ex:structure}
\end{myex}

Now Lemma~\ref{lemma:disjoint} and Proposition~\ref{prop:disjoint} are stated to back up a `disjointness' assumption that is made by COA upon minimal structures (Def.~\ref{def:minimal}). The proof we present here for Proposition~\ref{prop:disjoint} is a conveniently derived alternative to Simon's own proof to his original `theorem 3.2' \cite[p.~59]{simon1953}.

\begin{mylemma}
\label{lemma:disjoint}
Let $\mathcal S_1(\mathcal E_1, \mathcal V_1)$ and $\mathcal S_2(\mathcal E_2, \mathcal V_2)$ be structures. If $\;\mathcal V_1 \cap \mathcal V_2 = \varnothing$ then $\mathcal S_1 \cap \mathcal S_2 = \varnothing$ (i.e., $\mathcal E_1 \cap \mathcal E_2 = \varnothing$). That is, disjointness of variables is strong enough to warrant disjointness of equations.  
\end{mylemma}
\begin{proof}
Let $\mathcal V_1 \cap \mathcal V_2 = \varnothing$. Now by contradiction assume $\mathcal S_1 \cap \mathcal S_2 \neq \varnothing$, then there must be at least one shared equation $f \in \mathcal E_1, \mathcal E_2$. Since both $\mathcal S_1, \mathcal S_2$ are structures, by Def.~\ref{def:structure} we know that $|Vars(f)| \geq 1$ and $Vars(f) \subseteq \mathcal V_1 \cap \mathcal V_2$. Yet $\mathcal V_1 \cap \mathcal V_2 = \varnothing$. \lightning. Therefore $\mathcal V_1 \cap \mathcal V_2 = \varnothing$ implies $\mathcal S_1 \cap \mathcal S_2 = \varnothing$.  $\Box$
\end{proof} 
\begin{mydef}
\label{def:disjoint}
Let $\mathcal S_1(\mathcal E_1, \mathcal V_1)$ and $\mathcal S_2(\mathcal E_2, \mathcal V_2)$ be structures. Then we say that they are \textbf{disjoint} if $\;\mathcal V_1 \cap \mathcal V_2 = \varnothing$.
\end{mydef}
\begin{myprop}
\label{prop:disjoint}
Let $\mathcal S$ be a complete structure. If $\mathcal S_1, \mathcal S_2 \subset \mathcal S$ are any different minimal substructures of $\mathcal S$, then they are disjoint.  
\end{myprop}
\begin{proof}
We show the statement by case analysis and then contradiction out of Defs.~\ref{def:structure}--\ref{def:complete} and Defs.~\ref{def:minimal}--\ref{def:disjoint}. See \ref{app:disjoint}. $\Box$
\end{proof} 

Simon's COA is also based on a data structure introduced in Def.~\ref{def:matrix}.

\begin{mydef}
\label{def:matrix}
The \textbf{structure matrix} $A_S$ of a structure $\mathcal S(\mathcal E, \mathcal V)$, with $f_1,\, f_2, ..., f_n \in \mathcal E$ and $x_1,\, x_2, ..., x_m \in \mathcal V$, is a $|\mathcal E| \times |\mathcal V|$ matrix of 1's and 0's in which entry $a_{ij}$ is non-zero if variable $x_j$ appears in equation $f_i$, and zero otherwise.
\end{mydef}

Elementary row operations on the structure matrix may hinder the structure's causal ordering and then are not valid in general \cite{simon1953}. This also emphasizes that the problem of causal ordering is not about solving the system of equations of a structure, but identifying its hidden asymmetries.

\subsection{Simon's Causal Ordering Algorithm}\label{subsec:coa}
\noindent
Simon has described his Causal Ordering Algorithm (COA) only informally at a high level of abstraction \cite{simon1953}. It is given a complete structure $\mathcal S(\mathcal E, \mathcal V)$ and computes a causal mapping $\varphi$. The causal ordering itself is to be obtained as a post-processing (transitive closure) out of the causal mapping $\varphi$ and its induced set $C_\varphi$ of direct causal dependencies. 
Example~\ref{ex:coa} (continued) warms up for Simon's algorithm. 

\setcounter{myex}{0}
\begin{myex} (continued). 
Fig.~\ref{fig:coa-a} shows the matrix of the structure $\mathcal S$ given above in this example. By eliminating the variables identified with the minimal substructures $\mathcal{S}_1, \mathcal{S}_2, \mathcal{S}_3 \subset \mathcal S$, a smaller structure $\mathcal T$ is derived to be input at the next recursive step (see Fig.~\ref{fig:coa-b}). This is the main insight of Simon's to arrive at his recursive causal ordering algorithm, as described next. $\Box$ 
\label{ex:struct-matrix}
\end{myex}

\begin{spacing}{1.2}
\begin{figure}[t]\footnotesize
\tikzset{node style ge/.style={circle,inner sep=0pt,minimum size=10pt}}
\begin{subfigure}{0.29\textwidth}
\begin{tikzpicture}[baseline=(A.center)]
\matrix (A) [matrix of math nodes, nodes = {node style ge},column sep=0.6mm] {
 & \node (x1) {x_1}; & \node (x2) {x_2}; & \node (x3) {x_3}; & \node (x4) {x_4}; & \node (x5) {x_5}; & \node (x6) {x_6}; & \node (x7) {x_7};\\
\node (f1) {f_1}; & \node (a11) {1}; & \node (a12) {0}; & \node (a13) {0}; & \node (a14) {0}; & \node (a15) {0}; & \node (a16) {0}; & \node (a17) {0};\\
\node (f2) {f_2}; & \node (a21) {0}; & \node (a22) {1}; & \node (a23) {0}; & \node (a24) {0}; & \node (a25) {0}; & \node (a26) {0}; & \node (a27) {0};\\
\node (f3) {f_3}; & \node (a31) {0}; & \node (a32) {0}; & \node (a33) {1}; & \node (a34) {0}; & \node (a35) {0}; & \node (a36) {0}; & \node (a37) {0};\\
\node (f4) {f_4}; & \node (a41) {1}; & \node (a42) {1}; & \node (a43) {1}; & \node (a44) {1}; & \node (a45) {1}; & \node (a46) {0}; & \node (a47) {0};\\
\node (f5) {f_5}; & \node (a51) {1}; & \node (a52) {0}; & \node (a53) {1}; & \node (a54) {1}; & \node (a55) {1}; & \node (a56) {0}; & \node (a57) {0};\\
\node (f6) {f_6}; & \node (a61) {0}; & \node (a62) {0}; & \node (a63) {0}; & \node (a64) {1}; & \node (a65) {0}; & \node (a66) {1}; & \node (a67) {0};\\
\node (f7) {f_7}; & \node (a71) {0}; & \node (a72) {0}; & \node (a73) {0}; & \node (a74) {0}; & \node (a75) {1}; & \node (a76) {0}; & \node (a77) {1};\\
};
\end{tikzpicture}
\caption{\footnotesize{Structure matrix.}}\label{fig:coa-a}
\end{subfigure}
\tikzstyle{background0}=[rectangle,
                                                fill=gray!0,
                                                inner sep=0.025cm,
                                                rounded corners=1mm]
\tikzstyle{background1}=[rectangle,
                                                fill=gray!30,
                                                inner sep=0.025cm,
                                                rounded corners=1mm]
\tikzstyle{background2}=[rectangle,
                                                fill=gray!70,
                                                inner sep=0.025cm,
                                                rounded corners=1mm]
\begin{subfigure}{0.33\textwidth}
\hspace{2.5pt}
\begin{tikzpicture}[baseline=(A.center)]
  \tikzset{BarreStyle/.style =   {opacity=.35,line width=2.65 mm,line cap=round,color=#1}}
\matrix (A) [matrix of math nodes, nodes = {node style ge},column sep=0.6 mm] {
 & \node (x1) {x_1}; & \node (x2) {x_2}; & \node (x3) {x_3}; & \node (x4) {x_4}; & \node (x5) {x_5}; & \node (x6) {x_6}; & \node (x7) {x_7};\\
\node (f1) {f_1}; & \node (a11) {1}; & \node (a12) {0}; & \node (a13) {0}; & \node (a14) {0}; & \node (a15) {0}; & \node (a16) {0}; & \node (a17) {0};\\
\node (f2) {f_2}; & \node (a21) {0}; & \node (a22) {1}; & \node (a23) {0}; & \node (a24) {0}; & \node (a25) {0}; & \node (a26) {0}; & \node (a27) {0};\\
\node (f3) {f_3}; & \node (a31) {0}; & \node (a32) {0}; & \node (a33) {1}; & \node (a34) {0}; & \node (a35) {0}; & \node (a36) {0}; & \node (a37) {0};\\
\node (f4) {f_4}; & \node (a41) {1}; & \node (a42) {1}; & \node (a43) {1}; & \node (a44) {1}; & \node (a45) {1}; & \node (a46) {0}; & \node (a47) {0};\\
\node (f5) {f_5}; & \node (a51) {1}; & \node (a52) {0}; & \node (a53) {1}; & \node (a54) {1}; & \node (a55) {1}; & \node (a56) {0}; & \node (a57) {0};\\
\node (f6) {f_6}; & \node (a61) {0}; & \node (a62) {0}; & \node (a63) {0}; & \node (a64) {1}; & \node (a65) {0}; & \node (a66) {1}; & \node (a67) {0};\\
\node (f7) {f_7}; & \node (a71) {0}; & \node (a72) {0}; & \node (a73) {0}; & \node (a74) {0}; & \node (a75) {1}; & \node (a76) {0}; & \node (a77) {1};\\
};
 \draw [BarreStyle=blue] (a11.north west) to (a11.south east); 
 \draw [BarreStyle=red] (a22.north west) to (a22.south east);
 \draw [BarreStyle=green] (a33.north west) to (a33.south east);
 \draw [BarreStyle=blue] (a44.north west) to (a55.south east); 
 \draw [BarreStyle=red] (a66.north west) to (a66.south east); 
 \draw [BarreStyle=green] (a77.north west) to (a77.south east); 
     \begin{pgfonlayer}{background}
        \node [background0,
                    fit=(a11) (a12) (a13) (a14) (a15) (a16) (a17) (a21) (a22) (a23) (a24) (a25) (a26) (a27) (a31) (a32) (a33) (a34) (a35) (a36) (a37) (a41) (a42) (a43) (a44) (a45) (a46) (a47) (a51) (a52) (a53) (a54) (a55) (a56) (a57) (a61) (a62) (a63) (a64) (a65) (a66) (a67) (a71) (a72) (a73) (a74) (a75) (a76) (a77) ]
                    {};
        \node [background1,
                    fit=(a44) (a45) (a46) (a47) (a54) (a55) (a56) (a57) (a64) (a65) (a66) (a67) 
                    (a74) (a75) (a76) (a77) ]
                    {};
        \node [background2,
                    fit=(a66) (a67) (a76) (a77) ]
                    {};
    \end{pgfonlayer}
\end{tikzpicture}
\vspace{2pt}
\caption{\footnotesize{RTCM's run in 3 steps.}}
\label{fig:coa-b}
\end{subfigure}
\begin{subfigure}{0.33\textwidth}
\vspace{10pt}
\begin{center}
\tikzstyle{vertex}=[circle,fill=black!22,minimum size=16pt,inner sep=0pt]
\tikzstyle{selected vertex} = [vertex, fill=red!50]
\tikzstyle{edge} = [draw,thick,->,bend left]
\tikzstyle{weight} = [font=\small]
\tikzstyle{selected edge} = [draw,line width=5pt,-,red!50]
\tikzstyle{ignored edge} = [draw,line width=5pt,-,black!20]
\begin{tikzpicture}[scale=0.6]
    \foreach \pos/\name in {{(-0.5,2)/x_1}, {(2,2)/x_2}, {(4.5,2)/x_3}, 
                            		 {(0.7,0)/x_4}, {(3.3,0)/x_5},
		 			 {(1,-2)/x_6}, {(3,-2)/x_7}}
        \node[vertex] (\name) at \pos {$\name$};
    \draw[->] (x_1) to (x_4);
    \draw[->] (x_1) to (x_5);
    \draw[->] (x_2) to (x_4);
    \draw[->] (x_3) to (x_4);
    \draw[->] (x_3) to (x_5);
    \draw[->] (x_4) to[out=10,in=200] (x_5);
    \draw[->] (x_5) to[out=170,in=-20] (x_4);
    \draw[->] (x_4) to (x_6);
    \draw[->] (x_5) to (x_7);
\end{tikzpicture}
\end{center}
\caption{\footnotesize{Causal graph $G_{\varphi}$.}}
\label{fig:causal-graph}
\end{subfigure}
\normalsize
\caption{Simon's RTCM, the core procedure in COA. Fig. \ref{fig:coa-a}: a structure matrix given. Fig. \ref{fig:coa-b}: minimal substructures detected in each recursive step $k$ are highlighted in shades of gray and have their diagonal elements colored. Fig. \ref{fig:causal-graph}: Causal graph $G_{\varphi}$ induced by mapping $\varphi$ over structure $\mathcal S$. An edge connects a node $x_i$ towards a node $x_j$, with $x_i, x_j \in \mathcal V$, iff $x_i$ appears in the equation $f \in \mathcal E$ such that $\varphi(f)=x_j$. As the mapping $\varphi$ is not unique, accordingly the causal graph $G_\varphi$ is not either---e.g., consider $\varphi^\prime$ with $f_4 \mapsto x_5$ and $f_5 \mapsto x_4$. The induced graph $G_{\varphi^\prime}$ would have, e.g., a connection from $x_2$ to $x_5$ instead. Yet their graph transitive closure is the same, $tc(G_{\varphi})=tc(G_{\varphi^\prime})$, as we shall see in \S\ref{sec:nayak}. }
\label{fig:coa}
\end{figure}
\end{spacing}

\begin{spacing}{1.1}
\begin{algorithm}[h]\footnotesize
\caption{Simon's Causal Ordering Algorithm based on RTCM.}
\label{alg:coat}
\begin{algorithmic}[1]
\Procedure{\textsf{COA}}{${\mathcal S\!:\, \text{structure over}\; \mathcal E \;\text{and}\; \mathcal V}$}
\Require $\mathcal S$ given is complete, i.e., $|\mathcal E|=|\mathcal V|$
\Ensure Returns $C_\varphi^+$, the causal ordering of $\mathcal S$
\State $\varphi \gets \textsf{RTCM}(\mathcal S)$ \Comment{gets total causal mapping $\varphi$ by Simon's recursive algorithm}
\State $C_\varphi \gets \varnothing$ \Comment{initializes set of direct causal dependencies}
\ForAll{$\langle f, x \rangle \in \varphi$}
\ForAll{$x_a \in Vars(f) \setminus \{x\}$}
\State $C_\varphi \gets C_\varphi \cup \{(x_a, x)\}$
\EndFor
\EndFor
\State \Return $\textsf{TC}(C_\varphi)$ \Comment{returns the transitive closure of $C_\varphi$, as described in \S\ref{subsec:closure}}
\EndProcedure
\end{algorithmic}
\hrule
\begin{algorithmic}[1]
\Procedure{\textsf{RTCM}}{${\mathcal S\!:\, \text{structure over}\; \mathcal E \;\text{and}\; \mathcal V}$}
\Require Structure $\mathcal S$ given is complete, i.e., $|\mathcal E|=|\mathcal V|$
\Ensure Returns total causal mapping $\varphi: \mathcal E \to \mathcal V$
\State $\varphi \gets \varnothing$, $\mathcal S^\star \gets \varnothing$, $D \gets \varnothing$ \Comment{initializes}
\State identify all minimal substructures $\mathcal S^\prime \subseteq \mathcal S$
\ForAll{minimal $\mathcal S^\prime \subseteq \mathcal S$}\vspace{1pt}
\State $\mathcal S^\star \gets \mathcal S^\star \cup \mathcal S^\prime$ \Comment{aggregates into $\mathcal S^\star$ each minimal substructure scanned}\vspace{1pt}
\ForAll{$f \in \mathcal E^\prime$, where $\mathcal S^\prime$}
\State $x \gets \text{any} \;x_a$ such that $x_a \in Vars(f)$ and $x_a \notin D$\vspace{1pt}
\State $\varphi \gets \varphi \cup \langle f,\, x \rangle$ \Comment{maps to $f$ some variable $x \in Vars(f)$}\vspace{1pt}
\State $D \gets D \cup \{x\}$ \Comment{aggregates into $D$ the variables already `matched'}
\EndFor
\EndFor
\State $\mathcal T \gets \mathcal S \div \mathcal S^\star$ \Comment{removes $\mathcal E^\star$;  eliminates $\mathcal V^\star = \bigcup_{f \in \mathcal E^\star} Vars(f)$, where n.b., $\mathcal V^\star=D$ }\vspace{1pt}
\If{$\mathcal T \neq \varnothing$}
\State \Return $\varphi \;\cup\;$$\textsf{RTCM}$$(\mathcal T)$
\EndIf
\vspace{-3pt}
\State \Return $\varphi$
\EndProcedure
\end{algorithmic}
\end{algorithm}
\end{spacing}

Algorithm~\ref{alg:coat} describes the variant of Simon's original description that returns a `total' causal mapping (satisfies Def.~\ref{def:tcm}).\footnote{This slight variation takes place in lines 7--10 of RTCM in Algorithm \ref{alg:coat}, and is irrelevant to its intractability---which we shall see is due to line 3. Besides, `total' and `partial' causal mappings are interchangeable  \cite{druzdzel2008}. In particular, recovering the latter from the former is straightforward: just merge `strongly coupled' variables in a cluster. Intuitively, these are variables whose values can only be determined simultaneously. To be precise, let $x_1, x_2 \in \mathcal V$ be variables in a structure $\mathcal S(\mathcal E, \mathcal V)$. We say $x_1, x_2$ are \emph{strongly coupled} if $\mathcal S$ is minimal. } 
We refer to its core procedure as RTCM (recursive total causal mapping). It comprises the actual source of intractability in Simon's original description, while lines 3-7 of the COA procedure were not described by himself but only considered as matter of a post-processing. We illustrate RTCM through Example \ref{ex:coa} and Fig. \ref{fig:coa}.

\setcounter{myex}{0}
\begin{myex} (continued). Let $\mathcal T = \mathcal S \div (\mathcal S_1 \cup \mathcal S_2 \cup \mathcal S_3)$ be the structure returned by COA's first recursive step $k=0\,$ for this example. $\!$Then a valid total causal mapping that can be returned at $k=1$ (see Fig.\ref{fig:coa-b}) is COA$(\mathcal T) = \{ \langle f_4, x_4 \rangle,\, \langle f_5, x_5 \rangle\}$.  
Since $x_4$ and $x_5$ are strongly coupled, COA maps them arbitrarily (e.g., it could be $f_4 \mapsto x_5,\, f_5 \mapsto x_4$ instead). Such total causal mapping $\varphi$ renders a cycle in the directed causal graph $G_{\varphi}$ induced by $\varphi$ (see Fig.\ref{fig:causal-graph}), which might not be desirable for some applications. 
$\Box$
\label{ex:coa}
\end{myex}

\subsection{Hardness of Simon's Recursive COA}\label{subsec:hard}

\noindent
First of all, we state a decision problem associated with finding the minimal structures in a given structure (line 3 of Simon's RTCM procedure in Algorithm~\ref{alg:coat}). For short, we shall refer to this problem as the Complete Substructure Decision Problem (CSDP). 

\begin{framed}
\noindent
(CSDP). Given a complete structure $\mathcal S(\mathcal E, \mathcal V)$ with $|\mathcal E|=|\mathcal V|=m$ and an integer $1 \leq \ell < m$, does $\mathcal S$ have a complete substructure $\mathcal S^\prime(\mathcal E^\prime, \mathcal V^\prime)$ with $|\mathcal E^\prime|=|\mathcal V^\prime|=\ell$? 
\end{framed}

In this section we carry out an original study on CSDP and show that it is NP-Complete. We consider a basic observation by Nayak \cite{nayak1994} apud. \cite{serrano1987}, that there is a correspondence between Simon's structures and bipartite graphs. A graph is said \emph{bipartite} if its vertices can be divided into two disjoint sets $V_1$ and $V_2$ and every edge connects a vertex in $V_1$ to one in $V_2$ \cite{bondy1976}. Moreover it is said to be $\ell$-\emph{balanced} if $|V_1|=|V_2|=\ell$, and is said to be \emph{connected} if $deg(w) \geq 1$ for all $w \in V_1 \cup V_2$. 
Def.~\ref{def:bipartite} introduces the mentioned correspondence and provides us some shorthand notation. 

\begin{mydef}\label{def:bipartite}
Let $\mathcal S(\mathcal E, \mathcal V)$ be a structure, and $G=(V_1 \cup V_2, E)$ be a bipartite graph where $V_1 \mapsto \mathcal E$ and $V_2 \mapsto \mathcal V$, and $E \mapsto \mathcal S$ so that an edge $(f, x) \in E$ if and only if we have $x \in Vars(f)$. We say that $G$ is the bipartite graph that \textbf{corresponds to} structure $\mathcal S$, and for short write $G \sim \mathcal S$.
\end{mydef}

Fig. \ref{fig:bipartite} shows the bipartite graph $G \sim \mathcal S$, where $\mathcal S$ is the initial structure given in Example \ref{ex:struct-matrix}. 

Def.~\ref{def:structural} introduces a bipartite graph property of our interest, and then Lemma~\ref{lemma:bipartite} originally establishes an equivalence of two problems: searching for complete substructures $\mathcal S^\prime \subset \mathcal S$ and searching for specific bipartite subgraphs $G^\prime \subset G$. 
 
\begin{mydef}\label{def:structural}
Let $G(V_1 \cup V_2, E)$ be a bipartite graph. We say that $G$ is \textbf{structural} if, for every $V_1^\prime \subseteq V_1$, there is a connected bipartite subgraph $G^\prime(V_1^\prime \cup V_2^\prime, E^\prime)$ with $|V_1^\prime| \leq |V_2^\prime|$. (Note resemblance with Def.~\ref{def:structure}).
\end{mydef}

\begin{figure}[t]
\begin{center}
\tikzstyle{rect}=[rectangle,
                                    thick,
                                    minimum size=0.3cm,
                                    draw=black]
\tikzstyle{circ}=[circle,
                                    thick,
                                    minimum size=0.3cm,
                                    draw=black]
\tikzstyle{vertex}=[circle,fill=black!10,minimum size=20pt,inner sep=0pt]
\tikzstyle{selected vertex} = [vertex, fill=red!24]
\tikzstyle{edge} = [draw,thick,->,bend left]
\tikzstyle{weight} = [font=\small]
\tikzstyle{selected edge} = [draw,line width=5pt,-,red!50]
\tikzstyle{ignored edge} = [draw,line width=5pt,-,black!20]
\begin{tikzpicture}[scale=0.85]
    \foreach \pos/\name in {{(0,10)/f_1}, {(6,10)/x_1}, {(0,9)/f_2}, {(6,9)/x_2},
					{(0,8)/f_3}, {(6,8)/x_3}, {(0,7)/f_4}, {(6,7)/x_4},  
					{(0,6)/f_5}, {(6,6)/x_5}, {(0,5)/f_6}, {(6,5)/x_6}, 
		 			 {(0,4)/f_7}, {(6,4)/x_7}}
        \node[vertex] (\name) at \pos {$\name$};
    \draw[-] (f_1) to (x_1);
    \draw[-] (f_2) to (x_2);
    \draw[-] (f_3) to (x_3);
    \draw[-] (f_4) to (x_1);
    \draw[-] (f_4) to (x_2);
    \draw[-] (f_4) to (x_3);
    \draw[-] (f_4) to (x_4);
    \draw[-] (f_4) to (x_5);
    \draw[-] (f_5) to (x_1);
    \draw[-] (f_5) to (x_3);
    \draw[-] (f_5) to (x_4);
    \draw[-] (f_5) to (x_5);
    \draw[-] (f_6) to (x_4);
    \draw[-] (f_6) to (x_6);
    \draw[-] (f_7) to (x_5);
    \draw[-] (f_7) to (x_7);
\end{tikzpicture}
\end{center}
\vspace{-10pt}
\caption{Bipartite graph $G$ of structure $\mathcal S$ from Example \ref{ex:struct-matrix}.}\label{fig:bipartite}
\end{figure}
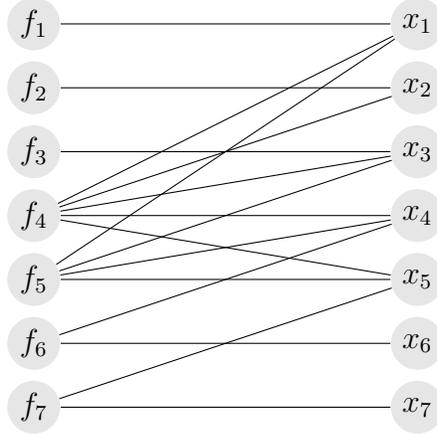
\vspace{0pt}

\begin{mylemma}\label{lemma:bipartite}
Let $\mathcal S(\mathcal V, \mathcal E)$ be a complete structure with $|\mathcal E|=|\mathcal V|=m$ and $1 \leq \ell < m$ provide an instance of CSDP. Let also $G(V_1 \cup V_2, E)$ be a bipartite graph $G \sim \mathcal S$. Then $\mathcal S$ has a substructure $\mathcal S^\prime$ that gives a yes answer to CSDP if and only if $G$ has a bipartite subgraph $G^\prime(V_1^\prime \cup V_2^\prime, E^\prime)$ such that $G^\prime \sim \mathcal S^\prime$ and all of these conditions hold: 
\begin{itemize}
\item[(i)]  Bipartite subgraph $G^\prime$ is structural; 
\item[(ii)] For every $f \in V_1^\prime$, there is an edge $(f, x) \in E$ only if $x \in V_2^\prime$; 
\item[(iii)]  Bipartite subgraph $G^\prime$ is $\ell$-balanced, that is, $|V_1^\prime| = |V_2^\prime|=\ell$;
\end{itemize}
\end{mylemma}
\begin{proof}
We establish conditions (i-iii) as the bipartite subgraph properties that correspond to a yes answer to CSDP. See \ref{app:coa-bipartite}.  
$\Box$
\end{proof}

We now reach the key property in our argument to show COA's hardness. A \emph{biclique} (or complete bipartite graph) is a bipartite graph $G=(V_1 \cup V_2,\, E)$ such that for every two vertices $u \in V_1$, $v \in V_2$, we have $(u,\, v) \in E$ \cite{even2011}. 
Thus the number of edges in a biclique is $|E| = |V_1| \cdot |V_2|$. A biclique with partitions of size $|V_1|=m\,$ and $|V_2|=n$ is denoted $K_{m, n}$. For instance, the bipartite graph $G$ shown in Fig.~\ref{fig:bipartite} has a 
$K_{2, 2}$ biclique, viz., $G^\prime(V_1^\prime \cup V_2^\prime, E^\prime)$, where $V_1^\prime = \{f_4, f_5\}$, $V_2^\prime = \{x_4, x_5\}$ and $E^\prime \!=\! \{(f_4, x_4), (f_4, x_5), (f_5, x_4), (f_5, x_5)\}$. 
%
Let us now consider Example~\ref{ex:hardness}.

\begin{spacing}{1.2}
\begin{figure}[t]
\tikzset{node style ge/.style={circle,inner sep=0pt,minimum size=16pt}}
\tikzstyle{background0}=[rectangle,
                                                fill=gray!17,
                                                inner sep=0.025cm,
                                                rounded corners=1mm]
\tikzstyle{background1}=[rectangle,
                                                fill=gray!65,
                                                inner sep=0.025cm,
                                                rounded corners=1mm]
\vspace{2pt}
\begin{subfigure}{0.33\textwidth}
\hspace{8pt}
\begin{tikzpicture}[baseline=(A.center)]
  \tikzset{BarreStyle/.style =   {opacity=.35,line width=3.25 mm,line cap=round,color=#1}}
\matrix (A) [matrix of math nodes, nodes = {node style ge},column sep=1.0 mm] {
 & \node (x1) {x_1}; & \node (x2) {x_2}; & \node (x3) {x_3}; & \node (x4) {x_4};\\
\node (f1) {f_1}; & \node (a11) {1}; & \node (a12) {0}; & \node (a13) {1}; & \node (a14) {0};\\
\node (f2) {f_2}; & \node (a21) {1}; & \node (a22) {1}; & \node (a23) {0}; & \node (a24) {0};\\
\node (f3) {f_3}; & \node (a31) {0}; & \node (a32) {1}; & \node (a33) {1}; & \node (a34) {0};\\
\node (f4) {f_4}; & \node (a41) {1}; & \node (a42) {1}; & \node (a43) {1}; & \node (a44) {1};\\
};
     \begin{pgfonlayer}{background}
        \node [background0,
                    fit=(a11) (a12) (a13) (a21) (a22) (a23) (a31) (a32) (a33) ]
                    {};
        \node [background1,
                    fit=(a44) ]
                    {};
    \end{pgfonlayer}
\end{tikzpicture}
\vspace{2pt}
\caption{\textsf{COA} $\!$(2$\!$ recursive steps).}\label{fig:hard-matrix}
\end{subfigure}
\hspace{-5pt}
\begin{subfigure}{0.32\textwidth}
\tikzstyle{circ}=[circle,
                                    thick,
                                    minimum size=0.2cm,
                                    draw=black]
\tikzstyle{vertex}=[circle,fill=black!10,minimum size=18pt,inner sep=0pt]
\tikzstyle{svertex} = [vertex, fill=black!30]
\tikzstyle{edge} = [draw,thick,->,bend left]
\tikzstyle{weight} = [font=\small]
\tikzstyle{selected edge} = [draw,line width=5pt,-,red!50]
\tikzstyle{ignored edge} = [draw,line width=5pt,-,black!20]
\begin{tikzpicture}[scale=0.85]
    \foreach \pos/\name in {{(-0.5,10)/f_1}, {(3.0,10)/x_1}, {(-0.5,9)/f_2}, {(3.0,9)/x_2},
					{(-0.5,8)/f_3}, {(3.0,8)/x_3}} 
        \node[vertex] (\name) at \pos {$\name$};
\node[svertex] (f_4) at (-0.5,7) {$f_4$};
\node[svertex] (x_4) at (3.0,7) {$x_4$};

    \draw[-] (f_1) to (x_1);
    \draw[-] (f_1) to (x_3);
    \draw[-] (f_2) to (x_1);
    \draw[-] (f_2) to (x_2);
    \draw[-] (f_3) to (x_2);
    \draw[-] (f_3) to (x_3);
    \draw[-] (f_4) to (x_1);
    \draw[-] (f_4) to (x_2);
    \draw[-] (f_4) to (x_3);
    \draw[-] (f_4) to (x_4);
\end{tikzpicture}
\vspace{2pt}
\caption{Bipartite graph $G$.}\label{fig:hard-graph}
\end{subfigure}
\hspace{-7pt}
\begin{subfigure}{0.35\textwidth}
\tikzstyle{circ}=[circle,
                                    thick,
                                    minimum size=0.2cm,
                                    draw=black]
\tikzstyle{vertex}=[circle,fill=black!10,minimum size=18pt,inner sep=0pt]
\tikzstyle{svertex} = [vertex, fill=black!30]
\tikzstyle{edge} = [draw,thick,->,bend left]
\tikzstyle{weight} = [font=\small]
\tikzstyle{selected edge} = [draw,line width=5pt,-,red!50]
\tikzstyle{ignored edge} = [draw,line width=5pt,-,black!20]
\begin{tikzpicture}[scale=0.85]
    \foreach \pos/\name in {{(3.0,10)/x_1}, {(3.0,9)/x_2},
					{(3.0,8)/x_3}} 
        \node[vertex] (\name) at \pos {$\name$};
\node[svertex] (f_1) at (-0.5,10) {$f_1$};
\node[svertex] (f_2) at (-0.5,9) {$f_2$};
\node[svertex] (f_3) at (-0.5,8) {$f_3$};
\node[vertex] (f_4) at (-0.5,7) {$f_4$};
\node[svertex] (x_4) at (3.0,7) {$x_4$};

    \draw[-] (f_1) to (x_2);
    \draw[-] (f_1) to (x_4);
    \draw[-] (f_2) to (x_3);
    \draw[-] (f_2) to (x_4);
    \draw[-] (f_3) to (x_1);
    \draw[-] (f_3) to (x_4);
\end{tikzpicture}
\vspace{2pt}
\caption{Bipartite complement $G^c\!$.}\label{fig:hard-cgraph}
\end{subfigure}
\caption{Another example of structure $\mathcal S$ with its correspondent bipartite graph $G \sim \mathcal S$.}
\label{fig:hardness}
\end{figure}
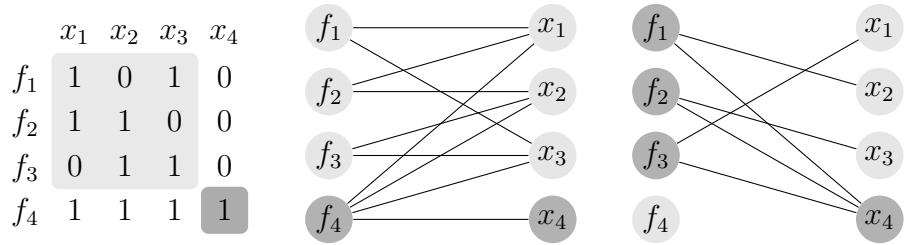
\end{spacing}

\begin{myex}\label{ex:hardness}
We introduce another structure $\mathcal S$, whose structure matrix is shown in Fig.~\ref{fig:hard-matrix} together with the bipartite graph $G \sim \mathcal S$ in Fig.~\ref{fig:hard-graph}. 
Let us consider subgraph $G^\prime(V_1^\prime \cup V_2^\prime, E^\prime)$ in $G$ that has $V_1^\prime=\{f_1, f_2, f_3\}$ and $V_2^\prime=\{x_1, x_2, x_3\}$. 
Observe that we have $G^\prime \sim \mathcal S^\prime$, where $\mathcal S^\prime \subset \mathcal S$ is the complete substructure represented by the shaded $3 \times 3$ matrix in Fig.~\ref{fig:hard-matrix}. 

Note also 
that such bipartite subgraph $G^\prime$ 
satisfies the conditions (i-iii) of Lemma~\ref{lemma:bipartite} and  in fact $\mathcal S^\prime$ is a complete substructure in $\mathcal S$. 
Clearly, $G^\prime \sim \mathcal S^\prime$ is not a biclique, as it is not the case that $deg(w) = 3$ for all $w \in V_1^\prime \cup V_2^\prime$. So there is no obvious connection between identifying complete substructures in a structure and bicliques in a bipartite graph. 
$\Box$
\end{myex}
\vspace{4pt}

The key insight to COA's hardness comes as follows---consider Example~\ref{ex:hardness} and Fig.~\ref{fig:hardness} for illustration. Recall from Lemma~\ref{lemma:bipartite}(ii) that, if we had an edge, say, connecting $(f_1, x_4) \in E$, then by Def.~\ref{def:structure} the substructure $\mathcal S^\prime(\mathcal E^\prime, \mathcal V^\prime)$ with $\mathcal E^\prime = \{f_1, f_2, f_3\}$ would have $\mathcal V^\prime = \bigcup_{f \in \mathcal E^\prime} Vars(f) = \{x_1, x_2, x_3, x_4\}$ instead. That is, it would no more be a complete substructure. 
In fact, verifying such a negative property (Lemma~\ref{lemma:bipartite}.ii) in structural bipartite graphs translates onto a neat positive property (biclique) in the bipartite complement of bipartite graph $G$. 

The \emph{bipartite complement} of a bipartite graph $G(V_1 \cup V_2 , E)$ is a bipartite graph $G^c(V_1 \cup V_2, E^c)$ where an edge $(u, v) \in E^c$ iff $(u, v) \notin E$ for every $u \in V_1$ and $v \in V_2$. Given a bipartite graph $G(V_1 \cup V_2, E)$, it is easy to see that we can render $G^c(V_1 \cup V_2, E^c)$ in polynomial time---consider, e.g., the biadjacency matrix of $G$ (viz., the structure matrix in Fig.~\ref{fig:hard-matrix}), and run a full scan on it to switch the boolean value of each entry in time $O(|V_1|\cdot |V_2|)$. Moreover, this operation is clearly invertible, as there is a one-to-one correspondence between $G$ and $G^c$. 

Fig.~\ref{fig:hard-cgraph} shows the bipartite complement graph $G^c$ of the bipartite graph $G$ from Fig.~\ref{fig:hard-graph}. 
Note that $G^c$ has a biclique $K_{3, 1}$ with its vertices shaded in dark grey. 
To emphasize the point, if we had an edge $(f_1, x_4) \in E$ (see Fig.~\ref{fig:hard-graph}), then such a biclique $K_{3, 1}$ would not exist in $G^c$ (see Fig.~\ref{fig:hard-cgraph}). We would have a $K_{2,\, 1}$ biclique instead with all edges in $\{f_2, f_3\} \times \{x_4\}$, but note that $2+1=3$ does not sum up to $|V_1|=|V_2|=m=4$. 

We can now establish the result we seek. We introduce below the Exact Node Cardinality Decision Problem (ENCD), which is a variant of biclique problem in bipartite graphs that is known to be NP-Complete \cite[p.~393]{dawande2001}. Theorem~\ref{thm:biclique} establishes its connection with CSDP. 

\begin{framed}
\noindent
$\!$(ENCD). Given a bipartite graph $G=(V_1 \cup V_2,\, E)$ and two positive integers $a, b$, does $G$ have a biclique $K_{a,\, b}$? 
\end{framed}

\begin{mythm}\label{thm:biclique}
CSDP is NP-Complete. 
\end{mythm}
\begin{proof}
We shall construct an instance of ENCD and describe its poly\-nomial-time reduction to an instance of CSDP. We refer to Def.~\ref{def:structural} and Lemma~\ref{lemma:bipartite} and present the argument in detail in \ref{app:coa-npcomplete}.  
$\Box$
\end{proof}

Finally, we formulate an optimization problem associated with CSDP. We refer to it as the Minimal Substructures Problem (MSP). Corollary~\ref{cor:coa-hardness} then finally establishes the hardness of Simon's COA based on RTCM. 

\begin{framed}
\noindent
(MSP). Given a complete structure $\mathcal S(\mathcal E, \mathcal V)$ with $|\mathcal E|=|\mathcal V|=m$, list all its complete substructures $\mathcal S^\prime(\mathcal E^\prime, \mathcal V^\prime)$ with $|\mathcal E^\prime|=|\mathcal V^\prime|=\ell$ where $1 \leq \ell < m$ is minimal.
\end{framed}

\begin{mycor}\label{cor:coa-hardness}
Let $\mathcal S$ be a complete structure. The extraction of its causal ordering by Simon's $COA(\mathcal S)$ through its RTCM procedure requires solving MSP, which is NP-Hard.
\end{mycor}
\begin{proof}
Clearly, MSP is the optimization problem that needs to be solved at each recursive step $k$ of Simon's RTCM procedure --- Algorithm~\ref{alg:coat}, line 3, ``find all minimal substructures $\mathcal S^\prime \subseteq \mathcal S$.'' 
But MSP is clearly an optimization problem that subsumes CSDP, which we know from Theorem~\ref{thm:biclique} that is NP-Complete by a reduction from ENCD.

In fact, an instance of ENCD$^{\,\prime}$ (as an optimization variant of ENCD) that can be reduced to MSP is as follows: given a bipartite graph $G(V_1 \cup V_2, E)$ that bears the complement structural property (cf. Theorem~\ref{thm:biclique}) and has $|V_1|=|V_2|=m$, list all bicliques $K_{\ell,\, m-\ell}$ contained in $G$ where $1 \leq \ell < m$ is minimal. 
In worst-case scenario, it requires searching for all bicliques $K_{\ell,\, m-\ell}$ for each of the $m-1$ possible values of $\,\ell$. 

ENCD is NP-Complete, therefore ENCD$^{\,\prime}$ is NP-Hard. 
Accordingly, CSDP is NP-Complete (cf. Theorem~\ref{thm:biclique}) therefore MSP is NP-Hard. 
$\Box$
\end{proof}

COP (Problem~\ref{prob:cop}), nonetheless, can be solved efficiently by means of a different approach due to Nayak \cite{nayak1994}, which we describe in next section.

\section{Nayak's Efficient Algorithm to COP}\label{sec:nayak}
\noindent
The first part of COP requires finding a total causal mapping $\varphi\!:\, \mathcal E \to \mathcal V$ over a given complete structure $\mathcal S$. While Simon's RTCM goes into an intractable step, inspired by Serrano and Gossard's work \cite{serrano1987} on constraint modeling and reasoning Nayak has found a polynomial-time approach to that task. 
We cover it next in all of its steps and see their complexity in some detail.

\subsection{Total Causal Mappings}\label{sec:tcm}

\noindent
Given a structure $\mathcal S$, there may be more than one total causal mappings over $\mathcal S$ (recall Example \ref{ex:struct-matrix}). So a question that arises is whether the transitive closure $C^+_{\varphi}$ is the same for any total causal mapping $\varphi$ over $\mathcal S$; that is, whether the causal ordering of $\mathcal S$ is unique. Proposition \ref{prop:causal-ordering}, from Nayak \cite{nayak1994}, ensures that is the case.

Before proceeding, we introduce Def.~\ref{def:strongly-coupled} in order to detach the notion of `strongly coupled' variables from `minimal structures' (Def.~\ref{def:minimal}) and connect it to the concept `causal dependency' (Def.~\ref{def:causal-dependency}).

\begin{mydef}\label{def:strongly-coupled}
Let $\mathcal S(\mathcal E, \mathcal V)$ be a structure with variables $x_a, x_b \in \mathcal V$, and $C^+_{\varphi}$ be the set of causal dependencies induced by total causal mapping $\varphi$ over $\mathcal S$. We say that $x_a$ and $x_b$ are \textbf{strongly coupled} if we have both $(x_a, x_b), (x_b, x_a) \in C_{\varphi}^+$.
\end{mydef}

Recall from Example~\ref{ex:coa} the strongly coupled variables, $x_4$ and $x_5$. Now we can see it only in terms of causal dependencies.

\begin{myprop}\label{prop:causal-ordering}
Let $\mathcal S(\mathcal E, \mathcal V)$ be a structure, and $\varphi_1\!:\, \mathcal E \to \mathcal V$ and $\varphi_2\!:\, \mathcal E \to \mathcal V$ be any two total causal mappings over $\mathcal S$. Then $C^+_{\varphi_1}$ = $C^+_{\varphi_2}$. 
\end{myprop}
\begin{proof}
The proof is based on an argument from Nayak \cite{nayak1994}, which we present in a bit more of detail (see \ref{coa-efficient}). Intuitively, it shows that if $\varphi_1$ and $\varphi_2$ differ in the variable an equation $f$ is mapped to, then such variables, viz., $\varphi_1(f)=x_a$ and $\varphi_2(f)=x_b$, must be causally dependent on each other (strongly coupled). 
$\Box$
\end{proof}

Another issue is concerned with the precise conditions under which total causal mappings exist (i.e., whether or not all variables in the equations can be causally determined). In fact, by Proposition \ref{prop:mapping-existence}, based on Nayak \cite{nayak1994} apud. Hall \cite[p. 135-7]{even2011}, we know that the existence condition holds if and only if the given structure is complete. 
We refer to Even \cite{even2011} to briefly introduce the additional graph-theoretic concepts that are necessary here: 

\begin{itemize}
\item A \emph{matching} in a graph is a subset of edges such that no two edges in the matching share a common node. 
\item A matching is said \emph{maximum} if no edge can be added to the matching (without hindering the matching property). 
\item Finally, a matching in a graph is said `perfect' if every vertex is an end-point of some edge in the matching --- in a bipartite graph, a perfect matching is said to be a \emph{complete} matching. 
\end{itemize}

\begin{myprop}\label{prop:mapping-existence}
Let $\mathcal S(\mathcal E, \mathcal V)$ be a structure. Then a total causal mapping $\varphi\!:\, \mathcal E \to \mathcal V$ over $\mathcal S$ exists if and only if $\mathcal S$ is complete. 
\end{myprop}
\begin{proof}
We observe that a total causal mapping $\varphi\!:\, \mathcal E \to \mathcal V$ over $\mathcal S$ corresponds exactly to a complete matching $M$ in a bipartite graph $B = (V_1 \cup V_2, E)$, where $V_1 \mapsto \mathcal E$, $V_2 \mapsto \mathcal V$, and $E \mapsto \mathcal S$.
In fact, by Even apud. Hall's theorem (cf. \cite[135-7]{even2011}), we know that $B$ has a complete matching iff (a) for every subset of vertices $F \subseteq V_1$, we have $|F| \leq |E(F)|$, where $E(F)$ is the set of all vertices connected to the vertices in $F$ by edges in $E$; and (b) $|V_1|=|V_2|$.  
By Def. \ref{def:structure} (no subset of equations has fewer variables than equations), and Def. \ref{def:complete} (number of equations is the same as number of variables), it is easy to see that conditions (a) and (b) above hold iff $\mathcal S$ is a complete structure. 
$\Box$
\end{proof}

The problem of finding a maximum matching is a well-studied algorithmic problem. The Hopcroft-Karp algorithm is a classical solution \cite{karp1973}, bounded in polynomial time by $O(\sqrt{|V_1|+|V_2|}\,|E|)$. It solves maximum matching in a bipartite graph efficiently as a problem of maximum flow in a network (cf. \cite[p. 135-7]{even2011}, or \cite[p. 763]{cormen2009}). That is, we can handle the problem of finding a total causal mapping $\varphi$ over a structure $\mathcal S$ (see Alg. \ref{alg:tcm}) by first translating it to the problem of maximum matching in a bipartite graph in time $O(|\mathcal S|)$. Then we can just apply the Hopcroft-Karp algorithm to get the matching and finally translate it back to the total causal mapping $\varphi$. This procedure has been suggested by Nayak in connection with his Proposition \ref{prop:mapping-existence} and its respective proof \cite{nayak1994}.

Fig. \ref{fig:matching} shows the complete matching found by the Hopcroft-Karp algorithm for the structure given in Example \ref{ex:struct-matrix}.

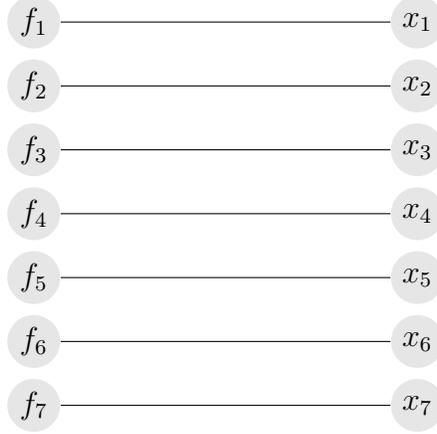
\begin{figure}[H]
\begin{center}
\tikzstyle{rect}=[rectangle,
                                    thick,
                                    minimum size=0.3cm,
                                    draw=black]
\tikzstyle{circ}=[circle,
                                    thick,
                                    minimum size=0.3cm,
                                    draw=black]
\tikzstyle{vertex}=[circle,fill=black!10,minimum size=20pt,inner sep=0pt]
\tikzstyle{selected vertex} = [vertex, fill=red!24]
\tikzstyle{edge} = [draw,thick,->,bend left]
\tikzstyle{weight} = [font=\small]
\tikzstyle{selected edge} = [draw,line width=5pt,-,red!50]
\tikzstyle{ignored edge} = [draw,line width=5pt,-,black!20]
\begin{tikzpicture}[scale=0.85]
    \foreach \pos/\name in {{(0,12)/f_1}, {(6,12)/x_1}, {(0,11)/f_2}, {(6,11)/x_2},
					{(0,10)/f_3}, {(6,10)/x_3}, {(0,9)/f_4}, {(6,9)/x_4},  
					{(0,8)/f_5}, {(6,8)/x_5}, {(0,7)/f_6}, {(6,7)/x_6}, 
		 			 {(0,6)/f_7}, {(6,6)/x_7}}
        \node[vertex] (\name) at \pos {$\name$};
    \draw[-] (f_1) to (x_1);
    \draw[-] (f_2) to (x_2);
    \draw[-] (f_3) to (x_3);
    \draw[-] (f_4) to (x_4);
    \draw[-] (f_5) to (x_5);
    \draw[-] (f_6) to (x_6);
    \draw[-] (f_7) to (x_7);
\end{tikzpicture}
\end{center}
\vspace{-5pt}
\caption{Complete matching $M$ for structure $S$ from Example \ref{ex:struct-matrix}.}\label{fig:matching}
\end{figure}

\begin{spacing}{1.1}
\begin{algorithm}[H]\footnotesize
\caption{Find a total causal mapping for a given structure.}
\label{alg:tcm}
\begin{algorithmic}[1]
\Procedure{TCM}{$\mathcal S\!:\, \text{structure over}\; \mathcal E \;\text{and}\; \mathcal V$}
\Require $\mathcal S$ given is a complete structure, i.e., $|\mathcal E|=|\mathcal V|$
\Ensure Returns a total causal mapping $\varphi$
\State $B(V_1 \cup V_2, E) \gets \varnothing$
\State $\varphi \gets \varnothing$
\ForAll{$\langle f, X\rangle \in \mathcal S$} \Comment{translates structure $\mathcal S$ to a bipartite graph $B$}\vspace{1pt}
\State $V_1 \gets V_1 \cup \{f\}$
\ForAll{$x \in X$}\vspace{1pt}
\State $V_2 \gets V_2 \cup \{x\}$
\State $E \gets E \cup \{(f, x)\}$
\EndFor
\EndFor
\State $M \gets \textsf{Hopcroft-Karp}(B)$\Comment{solves the maximum matching problem}
\ForAll{$(f, x) \in M$}\vspace{1pt}\Comment{translates the matching to a total causal mapping}
\State $\varphi \gets \varphi \cup \{\langle f, x\rangle\}$
\EndFor
\State \Return $\varphi$\vspace{1pt}
\EndProcedure
\end{algorithmic}
\end{algorithm}
\end{spacing}

\noindent
Corollary \ref{cor:tcm} and Remark \ref{rmk:correct-mapping} summarize the results presented in this note.

\begin{mycor}\label{cor:tcm}
Let $\mathcal S(\mathcal E, \mathcal V)$ be a complete structure. Then a total causal mapping $\varphi\!:\, \mathcal E \to \mathcal V$ over $\mathcal S$ can be found by (Alg. \ref{alg:tcm}) TCM in time that is bounded by $O(\sqrt{|\mathcal V|} \cdot |\mathcal S|)$. 
\end{mycor}
\begin{proof}
Let $B = (V_1 \cup V_2, E)$ be the bipartite graph corresponding to complete structure $\mathcal S$ given to TCM, where $V_1 \mapsto \mathcal E$, $V_2 \mapsto \mathcal V$, and $E \mapsto \mathcal S$. 
The translation of $\mathcal S$ into $B$ is done by a scan over it. This scan is of length $|\mathcal S| = |E|$. Note that number $|E|$ of edges rendered is precisely the length $|\mathcal S|$ of structure, where the denser the structure, the greater $|\mathcal S|$ is. The re-translation of the matching computed by internal procedure Hopcroft-Karp, in turn, is done at expense of $|\mathcal E| = |\mathcal V| \leq |\mathcal S|$. Thus, it is easy to see that TCM is dominated by the maximum matching algorithm Hopcroft-Karp, which is known to be $O(\sqrt{|V_1|+|V_2|}\cdot |E|)$, i.e., $O(\sqrt{|\mathcal E|+|\mathcal V|} \cdot |\mathcal S|)$. Since $\mathcal S$ is assumed complete, we have $|\mathcal E| \!=\! |\mathcal V|$ then $\sqrt{|\mathcal V|+|\mathcal V|} = \sqrt{2}\,\sqrt{|\mathcal V|}$. Therefore, \textsf{TCM} must have running time at most $O(\sqrt{|\mathcal V|}\cdot |\mathcal S|)$.
$\Box$
\end{proof}

\subsection{Computing Transitive Closure}\label{subsec:closure}
\noindent
Finally, recall that the set $C_{\varphi}$ of direct causal dependencies induced by a total causal mapping $\varphi$ over a given structure $\mathcal S(\mathcal E, \mathcal V)$ produces to the so-called `causal graph,' i.e., a directed graph (digraph) $G(V, E)$ where $V \mapsto \mathcal V$ and $E \mapsto C_{\varphi}$. So, computing set $C_{\varphi}^+$ of causal dependencies given $C_{\varphi}$ corresponds to computing transitive closure (reachability links) on $G$. Note that $|V|=|\mathcal V|$, and also note that $|E|=|C_{\varphi}|=|\mathcal S|\!-\!|\mathcal V|=O(|\mathcal S|)$. 

Classical algorithms for such task (e.g., Floyd-Warshall's) are bounded in time $O(|\mathcal V|^3)$ \cite[p.~697]{cormen2009}. Another way to do it is by discovering reachability links using either one of the popular graph traversal algorithms, breadth-first search or depth-first search (DFS) \cite[p.~603]{cormen2009}. Algorithm~\ref{alg:tc} describes DFS-based transitive closure over digraph $G(V, E)$. It runs in time $O(|V|\cdot |E|)$, which means $O(|\mathcal V|\cdot |\mathcal S|)$ for a complete structure $\mathcal S(\mathcal E, \mathcal V)$.

\begin{spacing}{1.1}
\begin{algorithm}[H]\footnotesize
\caption{DFS-based transitive closure.}
\label{alg:tc}
\begin{algorithmic}[1]
\Procedure{TC}{$\;G(V, E)\!: \text{digraph}$} \Comment{where $G$ is such that $V \mapsto \mathcal V$ and $E \mapsto C_{\varphi}$}
\State $E^+ \gets \varnothing$ 
\ForAll{$v \in V$} \Comment{for all vertices $v$ in digraph $G$}\vspace{1pt}
\State $D \gets \varnothing$ \Comment{initializes $D$}
\State $\text{DFS}(G,v, D)$ \Comment{discovers into $D$ all $u$, where $v$ is reachable from $u$}
\State $D \gets D \setminus \{v\}$ \Comment{enforces an irreflexive transitive closure}
\State $E^+ \gets \bigcup_{u \in D} \{(u, v)\} \cup E^+$ 
\EndFor
\State \Return $G^+(V, E^+)$\vspace{1pt}
\EndProcedure
\vspace{2pt}
\hrule
\vspace{2pt}
\Procedure{DFS}{$G\!:\, \text{digraph},\; v\!: \;\text{vertex},\; D\!: \text{global set of discovered vertices}$}
\State $D \gets D \cup \{v\}$ \Comment{labels $v$ as discovered}
\ForAll{$u$ where $(u, v) \in G$}\vspace{1pt}
\If{$u \notin D$} \Comment{vertex $u$ is not yet labeled as discovered}
\State $\text{DFS}(G,u, D)$
\EndIf
\EndFor
\State \Return \vspace{1pt}
\EndProcedure
\end{algorithmic}
\end{algorithm}
\end{spacing}

\vspace{-5pt}
\begin{myremark}\label{rmk:correct-mapping}
Let $\mathcal S(\mathcal E, \mathcal V)$ be a complete structure. Then we know (cf. Proposition \ref{prop:mapping-existence}) that a total causal mapping over $\mathcal S$ exists. Let it be defined $\varphi \triangleq TCM(\mathcal S)$, which can be done in $O(\sqrt{|\mathcal V|} \cdot |\mathcal S|)$. Then the causal ordering implicit in $\mathcal S$ can be correctly extracted (cf. Proposition \ref{prop:causal-ordering}) by computing $C^+_\varphi$, the set of causal dependencies induced by $\varphi$, in terms of graph transitive closure (TC). The latter is bounded in time by $O(|\mathcal V|\cdot |\mathcal S|)$, that is, the complexity of COP is dominated by TC.

In other words, the complete recipe to solve COP consists in replacing Simon's RTCM by Nayak's TCM in COA (Algorithm~\ref{alg:coat}). Transitive closure (TC) in turn is computed as described in Algorithm~\ref{alg:tc}.
$\Box$
\end{myremark}

\section{Conclusions}
\noindent
Causal ordering inference is a classical problem in the AI literature, and still relevant in light of modern applications \cite{druzdzel2008}, e.g., large-scale hypothesis management and analytics \cite{goncalves2015cise}. In this note we have shown that Simon's classical algorithm (COA) tries to address an NP-Hard problem; and then we have given a detailed account on the state-of-the-art algorithms for the causal ordering problem (COP, stated as Problem~\ref{prob:cop}). The key points are:

\begin{itemize}
\item By Theorem~\ref{thm:biclique} and Corollary~\ref{cor:coa-hardness}, we know (an original hardness result) that Simon's approach to COP requires solving an NP-Hard problem; 

\item From the seminal work of Simon \cite{simon1953} (cf. \S\ref{sec:problem}) and Nayak \cite{nayak1994} (cf. \S\ref{sec:nayak}, and Propositions \ref{prop:causal-ordering} and \ref{prop:mapping-existence}), an approach is conveyed to solve COP efficiently; 

\item By Corollary \ref{cor:tcm}, we know how to process a complete structure into a total causal mapping in time that is bounded by $O(\sqrt{|\mathcal V|}\cdot |\mathcal S|)$. This is a core step to solve COP, which Simon's COA in turn makes intractable. 

\item By Remark \ref{rmk:correct-mapping}, we know how to extract the \emph{causal ordering} of a complete structure in time $O(|\mathcal V|\cdot |\mathcal S|)$, that is, in sub-quadratic time on the structure density (number of variable appearances). 
The machinery of causal ordering is then suitable for processing very large structures. 
\end{itemize}

\section*{Acknowledgments}
\noindent
We thank three anonymous reviewers for their careful reading and sharp criticism on a previous version of this manuscript. This work has been supported by the Brazilian funding agencies CNPq (grants n$^o\!$ 141838/2011-6, 309494/2012-5) and FAPERJ (grants INCT-MACC E-26/170.030/2008, `Nota $\!$10' $\!$E-26/100.286/2013). We thank IBM for a Ph.D. Fellowship award. 


\bibliographystyle{elsarticle-num}
\bibliography{note}

\begin{thebibliography}{10}
\expandafter\ifx\csname url\endcsname\relax
  \def\url#1{\texttt{#1}}\fi
\expandafter\ifx\csname urlprefix\endcsname\relax\def\urlprefix{URL }\fi
\expandafter\ifx\csname href\endcsname\relax
  \def\href#1#2{#2} \def\path#1{#1}\fi

\bibitem{simon1953}
H.~Simon, Causal ordering and identifiability, In Hood \& Koopmans (eds.),
  Studies in Eco\-nometric Methods, Chapter 3, John Wiley \& Sons, 1953.

\bibitem{goncalves2014}
B.~Goncalves, F.~Porto, {$\Upsilon$-DB}: {M}anaging scientific hypotheses as
  uncertain data, PVLDB 7~(11) (2014) 959--62.

\bibitem{druzdzel2008}
D.~Dash, M.~J. Druzdzel, A note on the correctness of the causal ordering
  algorithm, Artificial Intelligence 172~(15) (2008) 1800--8.

\bibitem{nayak1994}
P.~P. Nayak, Causal approximations, Artificial Intelligence 70~(1-2) (1994)
  277--334.

\bibitem{simon1994}
Y.~Iwasaki, H.~A. Simon, Causality and model abstraction, Artificial
  Intelligence 67~(1) (1994) 143--194.

\bibitem{pearl2000}
J.~Pearl, Causality: {M}odels, Reasoning, and Inference, Cambridge Univ. Press,
  2000.

\bibitem{nayak1996}
P.~P. Nayak, Automated modelling of physical systems, Springer-Verlag, 1996.

\bibitem{goncalves2015cise}
B.~Gon\c{c}alves, Managing scientific hypotheses as data with support for
  predictive analytics, {IEEE} Computing in Science \& Eng. 17~(5) (2015)
  35--43.

\bibitem{haas2011}
P.~Haas, P.~Maglio, P.~Selinger$\!$, W.~Tan, $\!${Data} is dead... without
  what-if models, PVLDB 4~(12) (2011) 1486--9.

\bibitem{hunter2003}
P.~J. Hunter, T.~K. Borg, Integration from proteins to organs: the {P}hysiome
  {P}roject., Nat. Rev. Mol. Cell. Biol. 4~(3) (2003) 237--43.

\bibitem{hines2004}
M.~Hines, T.~Morse, M.~Migliore, N.~Carnevale, G.~Shepherd, Model{DB}: {A}
  database to support computational neuroscience, J. Comput. Neurosci. 17~(1)
  (2004) 7--11.

\bibitem{chelliah2013}
V.~Chelliah, C.~Laibe, N.~{Le Nov{\`e}re}, Bio{M}odels {D}atabase: {A}
  repository of mathematical models of biological processes, Method. Mol.
  Biol.~(1021) (2013) 189--99.

\bibitem{goncalves2015c}
B.~Gon\c{c}alves, Managing large-scale scientific hypotheses as uncertain and
  probabilistic data, Ph.D. thesis, National Laboratory for Scientific
  Computing (LNCC), available at \href{http://arxiv.org/abs/1501.05290}{CoRR
  abs/1501.05290}, Brazil (2015).

\bibitem{serrano1987}
D.~Serrano, D.~C. Gossard, Constraint management in conceptual design, in:
  Knowledge {B}ased {E}xpert {S}ystems in {E}ngineering: {P}lanning and Design,
  Computational Mechanics Publications, 1987, pp. 211--24.

\bibitem{bondy1976}
J.~Bondy, U.~Murty, Graph theory with applications, North-Holland Publishing
  Co., 1976.

\bibitem{even2011}
S.~Even, Graph algorithms, 2nd Edition, Cambridge Univ. Press, 2011.

\bibitem{dawande2001}
M.~Dawande, P.~Keskinocak, J.~M. Swaminathan, S.~Tayur, On bipartite and
  multipartite clique problems, J. Algorithms 41~(2) (2001) 388--403.

\bibitem{karp1973}
J.~E. Hopcroft, R.~M. Karp, An $n^{5/2}$ algorithm for maximum matchings in
  bipartite graphs, {SIAM} Journal on Computing 2~(4) (1973) 225--31.

\bibitem{cormen2009}
T.~H. Cormen, C.~E. Leiserson, R.~L. Rivest, C.~Stein, Introduction to
  Algorithms, 3rd Edition, The MIT Press, 2009.

\end{thebibliography}

\appendix

\section{Proof of Proposition~\ref{prop:disjoint}}\label{app:disjoint}
\noindent
\emph{Let $\mathcal S$ be a complete structure. If $\mathcal S_1, \mathcal S_2 \subset \mathcal S$ are any different minimal substructures of $\mathcal S$, then they are disjoint.}

\begin{proof}
We show the statement by case analysis and then contradiction out of Defs.~\ref{def:structure}--\ref{def:minimal}. By assumption both $\mathcal S_1, \mathcal S_2$ are minimal (hence complete). Let their size be $|\mathcal V_1|=|\mathcal E_1|=\ell$ and $|\mathcal V_2|=|\mathcal E_2|=m$. Let also $\ell \leq m$. The argument is analogous otherwise but it shall be convenient to keep a placeholder for the size of the smaller substructure (with no loss of generality). 

By Def.~\ref{def:minimal} (minimal structures), we know that $\mathcal S_1 \not\subseteq \mathcal S_2$ and $\mathcal S_1 \not\supseteq \mathcal S_2$. Now suppose $\mathcal S_1, \mathcal S_2$ are not disjoint. Then 
by Def.~\ref{def:disjoint} there must be at least one shared variable $x \in \mathcal V_1, \mathcal V_2$, and thus we must have $|\mathcal V_1 \cup \mathcal V_2| \leq \ell+m-1$. 

We can then proceed through case analysis by inquiring how many equations are shared by $\mathcal S_1, \mathcal S_2$. Since $\mathcal S_1$ is minimal with $|\mathcal E_1|=|\mathcal V_1|=\ell$ for $1 \leq \ell \leq m$, the number of equations that are shared with $\mathcal S_2$ could be any $0 \leq k < \ell$. (Note that the case of $k=\ell$ shared equations would lead to the more obvious contradiction that $\mathcal S_1 \subseteq \mathcal S_2$, even though $\mathcal S_2$ is minimal). 

Let us start with the case $\mathcal E_1 \cap \mathcal E_2 = \varnothing$ to illustrate the rationale in its simplest form. In this case, no equations are shared yet at least one variable is. Then we have $|\mathcal E_1 \cup \mathcal E_2| = \ell+m$, but $|\mathcal V_1 \cup \mathcal V_2| \leq \ell+m-1$. Since we have both $\mathcal S_1, \mathcal S_2 \subset \mathcal S$, in fact we have their sets of equations $\mathcal E_1, \mathcal E_2 \subset \mathcal E$ as well and then $\mathcal E_1 \cup \mathcal E_2 \subseteq \mathcal E$. Now, by Def.~\ref{def:structure} (valid structure), we know that in any subset of $k>0$ equations of $\mathcal S$, at least $k$ different variables must appear. 
But rather we have $|\mathcal E_1 \cup \mathcal E_2| = \ell+m$ and yet $|\mathcal V_1 \cup \mathcal V_2| \leq \ell+m-1$. That is, we reach a contradiction to Def.~\ref{def:structure}, viz., $|\mathcal E_1 \cup \mathcal E_2| > |\mathcal V_1 \cup \mathcal V_2|$. \lightning. 

The next case is when one equation is shared ($|\mathcal E_1 \,\cap\, \mathcal E_2| = 1$). 
Note that, if we had $|\mathcal E_1|=|\mathcal V_1|=\ell=1$ in particular then the only equation $f \in \mathcal E_1$ would have $|Vars(f)|=1$ and be shared with $\mathcal E_2$, making $\mathcal S_1 \subseteq \mathcal S_2$ even though $\mathcal S_2$ is assumed minimal. \lightning. We rather know that $|\mathcal E_1|=\ell \geq 2$. Also, note that we must have $|Vars(f)| \geq 2$ for all $f \in \mathcal E_1$, otherwise there would be some $g \in \mathcal E_1$ with $|Vars(g)|=1$ even though $|\mathcal E_1| \geq 2$. That is, we would have a minimal substructure within $\mathcal S_1$, although it is minimal. 

So, since one equation is shared and for all $f \in \mathcal E_1$ we have $|Vars(f)| \geq 2$, then at least two variables must be shared. Observe then that $|\mathcal E_1 \cup \mathcal E_2| = \ell+m-1$ (since exactly one equation is shared) and $|\mathcal V_1 \cup \mathcal V_2| \leq \ell+m-2$ (at least two variables are shared). Again, we see the same contradiction in face of Def.~\ref{def:structure}, viz., $|\mathcal E_1 \cup \mathcal E_2| > |\mathcal V_1 \cup \mathcal V_2|$. \lightning. 

Now we complete the case analysis by making the argument abstract for any number of shared equations, $0 \leq k < \ell$ (an inductive step, n.b., is not required because $k \in \mathbb N$ is bounded. Note that, for any such number $0 \leq k < \ell$, we must have at least $k+1$ shared variables, otherwise the shared substructure having $k$ equations, formed out of $\mathcal E_1 \cap \mathcal E_2$, would be minimal as well even though $\mathcal E_1 \cap \mathcal E_2 \subseteq \mathcal E_1, \mathcal E_2$ (that is, rendering both $\mathcal S_1, \mathcal S_2$ non-minimal. \lightning). However, once more we see that this contradicts Def.~\ref{def:structure}. \lightning.  $\Box$
\end{proof}


\section{Proof of Lemma \ref{lemma:bipartite}}\label{app:coa-bipartite}
\noindent
\emph{Let $\mathcal S(\mathcal V, \mathcal E)$ be a complete structure with $|\mathcal E|=|\mathcal V|=m$ and $1 \leq \ell < m$ provide an instance of CSDP. Let also $G(V_1 \cup V_2, E)$ be a bipartite graph $G \sim \mathcal S$. Then $\mathcal S$ has a substructure $\mathcal S^\prime$ that gives a yes answer to CSDP if and only if $G$ has a bipartite subgraph $G^\prime(V_1^\prime \cup V_2^\prime, E^\prime)$ such that $G^\prime \sim \mathcal S^\prime$ and all of these conditions hold: 
\begin{itemize}
\item[(i)]  Bipartite subgraph $G^\prime$ is structural; 
\item[(ii)] For every $f \in V_1^\prime$, there is an edge $(f, x) \in E$ only if $x \in V_2^\prime$;
\item[(iii)]  Bipartite subgraph $G^\prime$ is $\ell$-balanced, that is, $|V_1^\prime| = |V_2^\prime|=\ell$;
\end{itemize}
}

\begin{proof}
First, we consider the `if' statement---that is, all conditions (i-iii) together are sufficient. Let $G^\prime \subset G$ be a bipartite subgraph $G^\prime(V_1^\prime \cup V_2^\prime, E^\prime)$ that satisfies all conditions (i-iii), and $\mathcal S^\prime(\mathcal E^\prime, \mathcal V^\prime)$ be a substructure of $\mathcal S$ with $G^\prime \sim \mathcal S^\prime$. We shall see that such $\mathcal S^\prime$ does give a yes answer to CSDP, that is, it is a complete substructure with $|\mathcal E^\prime|=|\mathcal V^\prime|=\ell$. 

From condition (i) we know that $G^\prime$ is structural (Def.~\ref{def:structural}). That is, for every $V_1^{\prime\prime} \subseteq V_1^\prime$, there is a connected bipartite subgraph $G^{\prime\prime}(V_1^{\prime\prime} \cup V_2^{\prime\prime}, E^{\prime\prime})$ with $|V_1^{\prime\prime}| \leq |V_2^{\prime\prime}|$. Since $V_1^\prime \mapsto \mathcal E^\prime$, $V_2^\prime \mapsto \mathcal V^\prime$ and $E^\prime \mapsto \mathcal S^\prime$, such property bears obvious resemblance with Def.~\ref{def:structure} (structure). That is, the `connected' bipartite subgraph aspect implies that, for any subset of $|\mathcal E^{\prime\prime}|$ equations in $\mathcal E^\prime$, at least $|\mathcal V^{\prime\prime}| \geq |\mathcal E^{\prime\prime}|$ variables appear and each such variable $x \in \mathcal V^{\prime\prime}$ appears in some $f \in \mathcal E^{\prime\prime}$, otherwise $x \in V_2^{\prime\prime}$ would be disconnected in $G^{\prime\prime}(V_1^{\prime\prime} \cup V_2^{\prime\prime}, E^{\prime\prime})$. 

Condition (ii) ensures in addition that $\bigcup_{f \in \mathcal E^\prime} Vars(f) = \mathcal V^\prime$. That is, the variables in $\mathcal V^\prime$ are exhaustive w.r.t. $\mathcal E^\prime$. Thus, structure $\mathcal S^\prime$ satisfies Def.~\ref{def:structure}. Finally, condition (iii) ensures that $\mathcal S^\prime$ is complete with $|\mathcal E^\prime|=|\mathcal V^\prime|=\ell$. 

We consider now the `only if' statement---i.e., every condition (i-iii) is necessary. We assume a bipartite graph $G^\prime \sim \mathcal S^\prime$ and show that lacking any one such condition implies that $\mathcal S^\prime$ cannot be complete or cannot be a structure at all. First, it is easy to see that when condition (iii) does not hold for $G^\prime$ then a structure $\mathcal S^\prime$ with $G^\prime \sim \mathcal S^\prime$ cannot be complete. 

Now suppose condition (ii) does not for $G^\prime$. That is, there is some $f \in V_1^\prime$ that has incidence with some $x \in V_2 \setminus V_2^\prime$. Thus we have $V_1^\prime \mapsto \mathcal E^\prime$ and $V_2^\prime \mapsto \mathcal V^\prime$ but $\bigcup_{f \in \mathcal E^\prime} Vars(f) \neq \mathcal V^\prime$. Therefore either $\mathcal S^\prime$ does not satisfy Def.~\ref{def:structure} or we cannot actually have $G^\prime \sim \mathcal S^\prime$. \lightning. 

Finally, consider that $G^\prime$ is not structural (Def.~\ref{def:structural}). That is, there is some $V_1^{\prime\prime} \subseteq V_1^\prime$ such that no connected bipartite subgraph $G^{\prime\prime}(V_1^{\prime\prime} \cup V_2^{\prime\prime}, E^{\prime\prime})$ exists in $G^\prime$ with $|V_1^{\prime\prime}| \leq |V_2^{\prime\prime}|$.
Considering $G^\prime \sim \mathcal S^\prime$, that would imply for $\mathcal S^\prime(\mathcal E^\prime, \mathcal V^\prime)$ either an equation $f \in \mathcal E^\prime$ with no variables (a disconnected vertex $x \in V_1^\prime$), or a redundant subset of equations---number of equations is larger than number of variables appearing in it. Either conditions violate Def.~\ref{def:structure}, so $\mathcal S^\prime$ cannot be a structure even though $G^\prime \sim \mathcal S^\prime$. \lightning. 
$\Box$
\end{proof}

\section{Proof of Theorem \ref{thm:biclique}}\label{app:coa-npcomplete}
\noindent
\emph{CSDP is NP-Complete.}
\begin{proof}
We shall construct an instance of ENCD and describe its poly\-nomial-time reduction to an instance of CSDP by using Lemma~\ref{lemma:bipartite}.

\textbf{Constructing an instance of ENCD}. 
Let $G(V_1 \cup V_2, E)$ be a bipartite graph such that, for every $V_1^\prime \subseteq V_1$, there is a bipartite subgraph $G^\prime(V_1^\prime \cup V_2^\prime, E^\prime)$  with $|V_1^\prime| \leq |V_2^\prime|$ and $deg(f) < |V_2^\prime|$ for all $f \in V_1^\prime$. 
Note that this is the complement property of the structural bipartite graph property (see Def.~\ref{def:structural}). It is meant to ensure that the bipartite complement graph $G^c(V_1 \cup V_2, E^c)$ of $G$ is structural---satisfies Def.~\ref{def:structural}. That is, when we produce $G^c$, we know that it can possibly correspond to a structure $\mathcal S$ such that $G^c \sim \mathcal S$. 
Let also $G$ have $|V_1|=|V_2|=m$ in order to ensure that such structure $\mathcal S$ will be complete as well---recall that $\mathcal S$ given in CSDP is assumed complete indeed. 

Now let $G$ and an integer $1 \leq \ell < m$ provide an instance of ENCD for integers $a = \ell$ and $b = m - \ell$. That is, our problem is to decide whether $G$ has a biclique $K_{\ell,\, m-\ell}$. 
Imposing both of the above properties on $G$, n.b., incurs in no loss of generality w.r.t. ENCD as it does not open a pruning opportunity in searching for a biclique $K_{\ell,\, m-\ell}$ in 
powerset $\mathcal P(V_1 \times V_2)$.
Such a biclique $K_{\ell,\, m-\ell}$, if existing in $G$, can be denoted $C(V_1^\prime \cup V_2^\star, K)$, where $|V_1^\prime|=\ell$ and $|V_2^\star|=m-\ell$, and $K$ is a complete set of edges between $V_1^\prime$ and $V_2^\star$. Note also that $V_1^\prime \subset V_1$ and $V_2^\star \subset V_2$.  

\textbf{Production of an instance of CSDP from the ENCD one}. 
Let $G^c(V_1 \cup V_2, E^c)$ be the bipartite complement graph of $G$, where an edge $(f, x) \in E^c$ if and only if $(f, x) \notin E$ for $f \in V_1$ and $x \in V_2$. 
Clearly, bipartite graph $G^c$ can be produced in polynomial time from $G$---as mentioned in \S\ref{subsec:hard}, consider the `structure matrix' (biadjacency matrix) of $G$ and run a full scan on it to switch the boolean value of each entry in time $O(|V_1|\cdot |V_2|)$ and then get $G^c$. 
\textbf{Decision problem correspondence}. 
Now we show that a biclique $K_{\ell,\, m-\ell}$ in $G$, if existing, corresponds to a bipartite subgraph $G^{c\,^\prime}(V_1^\prime \cup V_2^\prime, E^{c\,\prime})$ in $G^c$ that satisfies the conditions (i-iii) of Lemma~\ref{lemma:bipartite}. That is, we show that a yes answer to ENCD implies a yes answer to CSDP. 

In fact, as $G^c$ is the bipartite complement graph of $G$, then the biclique $C(V_1^\prime \cup V_2^\star, K)$ in $G$ becomes a bipartite subgraph $C^c(V_1^\prime \cup V_2^\star, \varnothing)$ in $G^c$. 
Now let $G^{c\,\prime}(V_1^\prime \cup V_2^\prime, E^{c\,\prime})$ be such that $V_2^\prime = V_2 \setminus V_2^\star$. We observe that: 

\begin{itemize}
\item[(i)] The presence of biclique $C(V_1^\prime \cup V_2^\star, K)$ in $G$ indicates that $V_2^\star$ could not have contributed to satisfy the complement structural property for $V_1^\prime$, only $V_2^\prime = V_2 \setminus V_2^\star$ could. But such property turns into the structural property in $G^c$, thus $G^{c\,\prime}(V_1^\prime \cup V_2^\prime, E^{c\,\prime})$ must be structural indeed. That is, condition (i) of Lemma~\ref{lemma:bipartite} is ensured. 
\item[(ii)] By the fact that we have $C^c(V_1^\prime \cup V_2^\star, \varnothing)$ in $G^c$ we know that, 
for all $f \in V_1^\prime$, there can only be an edge $(f, x) \in E^c$ if $x \in V_2^\prime$ indeed. That is, condition (ii) of Lemma~\ref{lemma:bipartite} is ensured. 
\item[(iii)] The presence of biclique $C(V_1^\prime \cup V_2^\star, K)$ of form $K_{\ell,\, m-\ell}$ in $G$ implies that $V_1^\prime$ has size $|V_1^\prime|=\ell$. Besides, $V_2^\prime$ will have size $|V_2^\prime|=|V_2| - |V_2^\star| = m - (m-\ell) = \ell$. That is, we must have $|V_1^\prime| = |V_2^\prime| = \ell$ and then condition (iii) of Lemma~\ref{lemma:bipartite} is ensured as well. 
\end{itemize}

We have then established that the existence of a biclique $C \subset G$ of form $K_{\ell,\, m-\ell}$ implies the existence of a bipartite subgraph $G^{c\,\prime} \subset G^c$, where $G^{c\,\prime}$ satisfies the conditions (i-iii) of Lemma~\ref{lemma:bipartite}. That is, we get a yes answer to CSDP if we find one to ENCD. It remains to show the `only if' part of the correspondence. 

In fact, suppose no biclique $C(V_1^\prime \cup V_2^\star, K)$ of form $K_{\ell,\, m-\ell}$ exists in $G(V_1 \cup V_2, E)$. Clearly, it means that for any $V_1^\prime \subset V_1$ where $|V_1^\prime|=\ell$, there is at least one $f \in V_1^\prime$ such that an edge $(f, x)$ with $x \in V_2^\star$ is missing from $E$. Accordingly, in $G^c(V_1 \cup V_2, E^c)$, we cannot have $G^{c\,\prime} \subset G^c$ with condition (ii) of Lemma~\ref{lemma:bipartite} satisfied. 

ENCD is NP-Complete. Thus CSDP must be NP-Complete as well. 
$\Box$
\end{proof}

\section{Proof of Proposition \ref{prop:causal-ordering}}\label{coa-efficient}
\noindent
\emph{Let $\mathcal S(\mathcal E, \mathcal V)$ be a structure, and $\varphi_1\!:\, \mathcal E \to \mathcal V$ and $\varphi_2\!:\, \mathcal E \to \mathcal V$ be any two total causal mappings over $\mathcal S$. Then $C^+_1$ = $C^+_2$.} 
\begin{proof}
The proof is based on an argument from Nayak \cite{nayak1994}, which we reproduce here in a bit more of detail. Intuitively, it shows that if $\varphi_1$ and $\varphi_2$ differ in the variable an equation $f$ is mapped to, then such variables, viz., $\varphi_1(f)$ and $\varphi_2(f)$, must be causally dependent on each other (strongly coupled). 
 
To show $C^+_1$ = $C^+_2$ reduces to show both $C^+_1 \subseteq C^+_2$ and $C^+_2 \subseteq C^+_1$. We show the first containment, and the second is understood as following by symmetry. Closure operators are extensive, $X \subseteq cl(X)$, and 
idempotent, $cl(cl(X)) = cl(X)$. That is, if we have $C_1 \subseteq C_2^+$, then we shall have $C_1^+ \subseteq (C_2^+)^+$ and, by idempotence, $C_1^+ \subseteq C_2^+$. 

Then it suffices to show that $C_1 \subseteq C_2^+$, i.e., for any $(x^\prime,\, x) \in C_1$, we must show that $(x^\prime,\, x) \in C_2^+$ as well. Observe by Def. \ref{def:tcm} that both $\varphi_1$ and $\varphi_2$ are bijections, then, invertible functions. If $\varphi_1^{-1}(x) = \varphi_2^{-1}(x)$, then we have $(x^\prime,\, x) \in C_2$ and thus, trivially, $(x^\prime,\, x) \in C_2^+$. Else, $\varphi_1$ and $\varphi_2$ disagree in which equations they map onto $x$. But we show next, in any case, that we shall have $(x^\prime,\, x) \in C_2^+$. 

Take all equations $g \in \mathcal E^\prime \subseteq \mathcal E$ such that $\varphi_1(g) \neq \varphi_2(g)$, and let $n \leq |\mathcal E|$ be the number of such `disagreed' equations. Now, let $f \in \mathcal E^\prime$ be such that its mapped variable is $x = \varphi_1(f)$. 
Construct a sequence of length $2n$ such that, $s_0 = \varphi_1(f) = x$ and, for $1 \leq i \leq 2n$, element $s_i$ is defined $s_i = \varphi_2(\varphi_1^{-1}(s_{i-1}))$. That is, we are defining the sequence such that, for each equation $g \in \mathcal E^\prime$, its disagreed mappings $\varphi_1(g)=x_a$ and $\varphi_2(g)=x_b$ are such that $\varphi_1(g)$ is immediately followed by $\varphi_2(g)$. As $x_a,\, x_b \in Vars(g)$, we have $(x_a,\, x_b) \in C_2$ and, symmetrically, $(x_b,\, x_a) \in C_1$. The sequence is of form $s=\langle \underbrace{x,\, x_f}_{f}, \hdots, \underbrace{x_a,\, x_b}_{g}, \hdots, \underbrace{x_{2n-1},\, x_{2n}}_{h} \rangle$.

Since $x$ must be in the codomain of $\varphi_2$, we must have a repetition of $x$ at some point $2 \leq k \leq 2n$ in the sequence index, with $s_k=x$ and $s_{k-1}=x^{\prime\prime}$ such that $(x^{\prime\prime},\, x) \in C_2$. If $x^{\prime\prime}=x^\prime$, then $(x^\prime,\, x) \in C_2$ and obviously $(x^\prime,\, x) \in C_2^+$. Else, note that $x_f$ must also be in the codomain of $\varphi_1$, while $x^{\prime\prime}$ in the codomain of $\varphi_2$. Let $\ell$ be the point in the sequence, $3 \leq \ell \leq 2n\!-\!1$, at which $s_\ell=x_f=x_a$ and $s_{\ell+1}=x_b$ for some $x_b$ such that $(x_f,\, x_b) \in C_2$. It is easy to see that, either we have $x_b=x^{\prime\prime}$ or $x_b \neq x^{\prime\prime}$ but $(x_b,\, x^{\prime\prime}) \in C_2^+$. Thus, by transitivity on such a causal chain, we must have $(x_f,\, x^{\prime\prime}) \in C_2^+$ and eventually $(x_f,\, x) \in C_2^+$. Finally, since $x^\prime \in Vars(f)$ and $\varphi_2(f)=x_f$, we have $(x^\prime,\, x_f) \in C_2$ and, by transitivity, $(x^\prime,\, x) \in C_2^+$. 
$\Box$
\end{proof}

\end{document}